\documentclass[runningheads]{llncs}

\usepackage{graphicx}
\usepackage{color}
\usepackage{cite}
\usepackage{amsmath, amssymb}
\usepackage{xfrac}
\usepackage{verbatim}

\DeclareMathOperator*{\argmin}{arg\,min}

\numberwithin{equation}{section}
\begin{document}
\title{A Formalization of The Natural Gradient Method for General Similarity Measures}
\titlerunning{Formal Natural Gradient}
\author{Anton Mallasto \and Tom Dela Haije \and
Aasa Feragen}
\authorrunning{A. Mallasto et al.}
\institute{University of Copenhagen, Copenhagen, Denmark\\
\email{\{mallasto, haije, aasa\}@di.ku.dk}}
\maketitle              % typeset the header of the contribution
\begin{abstract}
In optimization, the natural gradient method is well-known for likelihood maximization. The method uses the Kullback-Leibler divergence, corresponding infinitesimally to the Fisher-Rao metric, which is pulled back to the parameter space of a family of probability distributions. This way, gradients with respect to the parameters respect the Fisher-Rao geometry of the space of distributions, which might differ vastly from the standard Euclidean geometry of the parameter space, often leading to faster convergence. However, when minimizing an arbitrary similarity measure between distributions, it is generally unclear which metric to use.
We provide a general framework that, given a similarity measure, derives a metric for the natural gradient. We then discuss connections between the natural gradient method and multiple other optimization techniques in the literature. Finally, we provide computations of the formal natural gradient to show overlap with well-known cases and to compute natural gradients in novel frameworks.

\keywords{Optimization  \and Natural Gradient \and Statistical Manifolds.}
\end{abstract}
\section{Introduction}
\emph{The natural gradient method} \cite{amari98} in optimization originates from \emph{information geometry}\cite{amari16}, which utilizes the Riemannian geometry of statistical manifolds (the parameter spaces of model families) endowed with the \emph{Fisher-Rao metric}. The natural gradient is used for minimizing the \emph{Kullback-Leibler} (KL) divergence, a \emph{similarity measure} between a model distribution and a target distribution,  that can be shown to be equivalent to maximizing model likelihood of given data. The success of natural gradient in optimization stems from accelerating likelihood maximization and providing infinitesimal invariance to reparametrizations of the model, providing robustness towards arbitrary parametrization choices.

In the modern formulation of the natural gradient, a \emph{Riemannian metric} on the statistical manifold is chosen, with respect to which the gradient of the given similarity is computed \cite[Sec. 12]{amari16}. However, it is generally unclear how to choose the Riemannian metric for a given similarity. One approach is pulling back our favorite metric from the space of distributions (e.g. the Fisher-Rao metric) to the statistical manifold, with no guarantees of the metric relating to the similarity that is being minimized. For example, see Fig.~\ref{toy_example}, where model selection for Gaussian process regression is carried out by maximizing the prior-likelihood of the data with natural gradients stemming from different metrics. Clearly, the Fisher-Rao metric---which infinitesimally corresponds to KL-divergence---achieves the fastest convergence. Motivated by this, we show how a natural Riemannian metric can be derived by locally approximating the Hessian matrix of the cost. We name the resulting approach the \emph{formal natural gradient} method.

Sometimes, one can compute the Hessian (Newton's method) or its \emph{absolute value} \cite{dauphin14} to derive such a metric. However, in many cases the Hessian can only be computed \emph{locally}, which is employed by the formal natural gradient.

\begin{figure}
\centering
\includegraphics[width=0.9\textwidth]{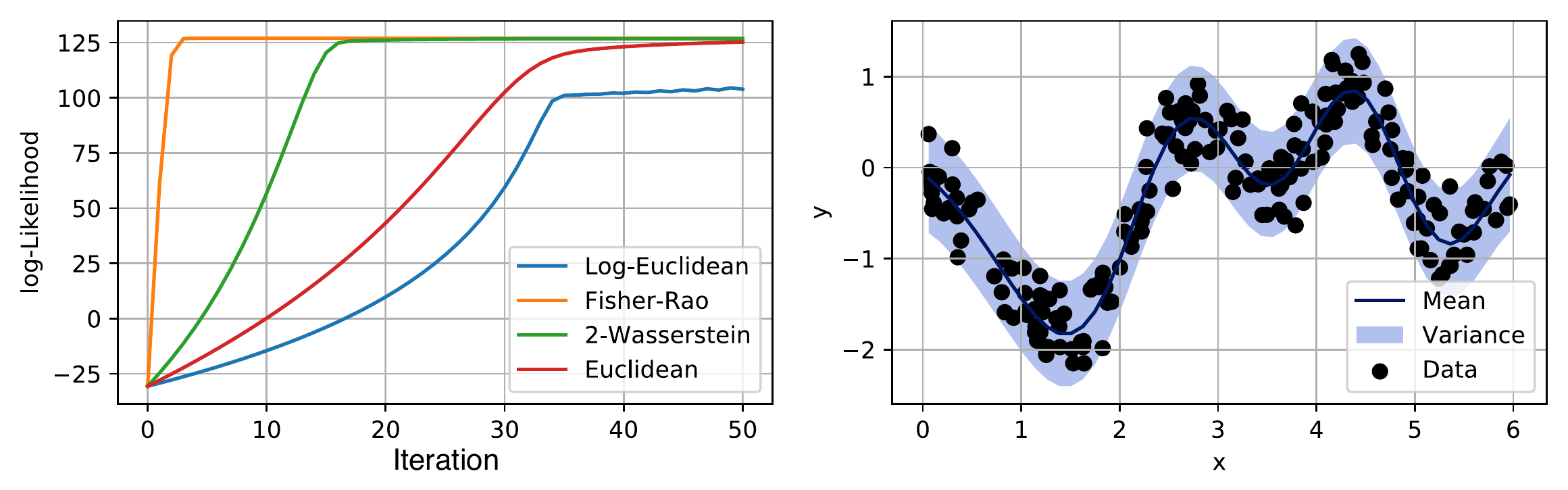}
\vspace{-0.15in}
\caption{Maximizing prior likelihood for Gaussian process regression using natural gradients under different metrics on Gaussian distributions. Convergence plots on left. Data and model fit, with optimal exponentiated quadratic kernel parameters, on right. }
\label{toy_example}
\end{figure}

\vspace{-0.4in}
%-----------------------------------------------------------------------------------------------------------------
\section{Useful Metrics via Formalizing the Natural Gradient}
\label{sec:formal_natural_gradient}
The natural gradient is computed with respect to a chosen metric on the statistical manifold, which often results from pulling back a metric between distributions. This way, the gradient takes into account how the metric on distributions penalizes movement into different directions. We will now review how the natural gradient is computed given a Riemannian metric. Then, we introduce the formal natural gradient, which derives the metric from the similarity measure.

\vspace{10pt}\hspace{-15pt}\textbf{Statistical manifold.} Let $\mathrm{AC}(X)$ denote the set of absolutely continuous probability densities on some manifold $X$. A \emph{statistical manifold} is defined by a triple $(X, \Theta, \rho)$, where $X$ is called the \emph{sample space} and $\Theta\subseteq \mathbb{R}^n$ the \emph{parameter space}. Then,  $\rho: \Theta \to \mathrm{AC}(X)$ is a $C^2$ map, given by $\rho: \theta \mapsto \rho_\theta(\cdot)$, for any $\theta \in \Theta$. Abusing terminology, we also call $\Theta$ the statistical manifold.

\vspace{10pt}\hspace{-15pt}\textbf{Cost function.} Let a \emph{similarity measure} $c^* \colon\mathrm{AC}(X) \times \mathrm{AC}(X) \to \mathbb{R}_{\geq 0}$  (e.g. a metric or an information divergence) be defined on $\mathrm{AC}(X)$ satisfying $c^*(\rho, \rho')=0$ if and only if $\rho=\rho'$. Assume $c^*$ is $C^2$ when $\rho\neq \rho'$ and strictly convex in $\rho$. Given a target distribution $\rho\in \mathrm{AC}(X)$ and a statistical manifold $(X, \Theta, \rho)$, we wish to minimize the \emph{cost function} $c:\Theta \times \mathrm{AC}(X) \to \mathbb{R}_{\geq 0}$ given by
\begin{equation}
c(\theta, \rho) = c^*(\rho_\theta, \rho).
\end{equation}
If $\rho = \rho_{\theta'}$ for some $\theta' \in \Theta$, then by abuse of notation we write $c(\theta, \theta')$.

\vspace{10pt}\hspace{-15pt}\textbf{Natural gradient.} Assume a Riemannian structure $(\Theta, g^\Theta)$ on the statistical manifold. The \emph{Riemannian metric} $g^\Theta$ induces a \emph{metric tensor} $G^\Theta$, given by $g_\theta^\Theta(u,v) = u^TG_\theta^\Theta v$ and a \emph{distance function} which we denote by $d_\Theta$. The vectors $u,v$  belong to the \emph{tangent space} $T_\theta \Theta$ at $\theta$. It is common intuition that the negative gradient $v = -\nabla_\theta c(\theta, \rho)$ gives the direction of maximal descent for $c$. However, this is only true on a Euclidean manifold. Consider 
\begin{equation}
\hat{v} = \argmin\limits_{v \in T_\theta \Theta:d_\Theta(\theta, \theta + v)= \Delta}  c(\theta + v, \rho),
\end{equation}
Linearly approximating the objective and quadratically approximating the constraint, this is solved using Lagrangian multipliers, giving the \emph{natural gradient}
\begin{equation}
\hat{v} = - \frac{1}{\lambda}\left[G_\theta^\Theta\right]^{-1}\nabla_\theta c(\theta, \rho),
\end{equation}
for some Lagrangrian multiplier $\lambda > 0$, which we refer to as the \emph{learning rate}.

\vspace{10pt}\hspace{-15pt}\textbf{Formal natural gradient.}  Traditionally, the natural gradient uses the Fisher-Rao metric when the similarity measure used is the KL-divergence. However, the authors are unaware of a general formal framework guiding the choice of the Riemannian metric given any similarity measure. To this end, we introduce the \emph{formal natural gradient}. Thus, consider the minimization task
\begin{equation}
\hat{v}:= \argmin_{v\in T_\theta \Theta,~c(\theta + v, \theta) = \Delta} c(\theta + v, \rho).
\label{def:naturalgradient}
\end{equation}
We approximate the constraint by the second degree Taylor expansion
\begin{equation}
c(\theta + v, \theta) \approx \frac{1}{2}v^T \left(\nabla^2_{\eta \to \theta}c(\eta, \theta)\right)v,
\label{eq:reg_appr}
\end{equation}
where the 0\textsuperscript{th} and 1\textsuperscript{st} degree terms disappear as $c(\theta + v,\theta)$ has a minimum 0 at $v=0$. We call the symmetric positive definite matrix $H^c_\theta :=\nabla^2_{\eta \to \theta}c(\eta, \theta)$ the \emph{local Hessian}. Then, we further approximate the objective function
\begin{equation}
c(\theta + v, \rho) \approx c(\theta, \rho) + \nabla_{\theta} c(\theta, \rho)^Tv,
\end{equation}
and thus by the method of Lagrangian multipliers, \eqref{def:naturalgradient} is solved as
\begin{equation}
\hat{v} = -\frac{1}{\lambda}\left[H^c_\theta\right]^{-1}\nabla_{\theta}c(\theta, \rho),
\label{eq:nat_grad}
\end{equation}
 We refer to $\hat{v}$ as the \emph{formal natural gradient} with respect to $c$.

\begin{remark}
We could have just substituted $\eta = \theta$ in the local Hessian if $\nabla^2_\eta c(\eta, \theta)$ was continuous at $\eta$. However, when studying Finsler metrics later in this work, the expression has a discontinuity at $\eta = \theta$. Therefore, we compute the limit from the direction of the gradient.
\end{remark}

\vspace{10pt}\hspace{-15pt}\textbf{Metric interpretation.} Fixing $\rho$ in \eqref{def:naturalgradient}, the local Hessian $G^c_\theta$ can be seen as a metric tensor at any $\theta\in \Theta$, inducing an inner product $g_\theta^c \colon T_\theta \Theta\times T_\theta \Theta \rightarrow \mathbb{R}$ given by $g_\theta^c(v,u) = v^TH^c_\theta u$. This imposes a \emph{pseudo-Riemannian} structure on $\Theta$, forming the pseudo-Riemannian manifold $(\Theta, g^c)$. Therefore, $G^c_x$ provides us a natural metric under which to compute the natural gradient for a general $c^*$. If $\rho$ has a full rank Jacobian everywhere, then a Riemannian metric is retrieved. Also, there is an obvious \emph{pullback} structure at play. Recall, that the cost is defined by $c(\theta, \theta')=c^*(\rho_{\theta}, \rho_{\theta'})$. Then, computing the local Hessian yields
\begin{equation}
H^c_\theta = J_\theta^T H^{c^*}_{\rho_\theta} J_\theta,
\end{equation}
where $H^{c^*}_{\rho_\theta} = \nabla^2_{\rho \to \rho_\theta} c^*(\rho, \rho_\theta)$. Thus, $H^c$ results from pulling back the $c^*$ induced metric tensor $H^{c^*}$ on $\mathrm{AC}(X)$ to the statistical manifold $\Theta$.

\vspace{10pt}\hspace{-15pt}\textbf{Asymptotically Newton's method.}
We provide a straightforward result, stating that the local Hessian approaches the actual Hessian in the limit, thus the formal natural gradient method approaches Newton's method.
\begin{proposition}
Assume $c(\theta,\rho) = c(\theta, \theta')$ for some $\theta' \in \Theta$, and that $c$ is $C^2$ in $\theta$. Then, the natural gradient yields asymptotically Newton's method.
\end{proposition}
\begin{proof}
The Hessian at $\theta$ is given by $\nabla^2_\theta c(\theta, \theta')$. Then, as $c$ is $C^2$ in the first argument, passing the limit $\theta\to \theta'$ yields
\begin{equation}
H^c_\theta = \nabla^2_{\eta \to \theta} c(\eta, \theta) \overset{\theta\to \theta'}{\to} \nabla^2_{\eta \to \theta'}c(\eta,\theta') = \nabla^2_{\eta = \theta'} c(\eta, \theta'),
\end{equation}
where the last expression is the Hessian at $\theta'$.
\end{proof}

\section{Loved Child Has Many Names -- Related Methods}
In this section, we discuss connections between seemingly different optimization methods. Some of these connections have already been reported in the literature, some are likely to be known to some extent in the community. However, the authors are unaware of previous work drawing out these connections in their full extent. We provide such a discussion, and then present other related connections.

As discussed in \cite{parikh14}, \emph{proximal methods} and \emph{trust region methods} are equivalent up to learning rate. Trust region methods employ an $l^2$-metric constraint
\begin{equation}
x_{t+1} = \argmin\limits_{x:\|x-x_t\|_2 \leq \Delta} f(x),~\Delta>0,
\label{trustregion}
\end{equation}
where as proximal methods include a $l^2$-metric penalization term
\begin{equation}
x_{t+1} =  \argmin\limits_{x}\left\lbrace f(x) + \frac{1}{2\lambda}\|x - x_t\|_2^2\right\rbrace,~\lambda > 0,
\end{equation}
The two can be shown to be equivalent up to learning rate via Lagrangian duality.

Instead of the $l^2$ metric penalization, \emph{mirror gradient descent}\cite{nemirovsky83} employs a more general \emph{proximity function} $\Psi:\mathbb{R}^n \times \mathbb{R}^n \to \mathbb{R}_{>0}$, that is strictly convex in the first argument. Then, the mirror descent step is given by
\begin{equation}
x_{t+1} = \argmin\limits_{x}\left\lbrace \langle x - x_t, \nabla f(x_t)\rangle + \frac{1}{\lambda}\Psi(x, x_t)\right\rbrace.
\end{equation}
Commonly, $\Psi$ is chosen to be a \emph{Bregman divergence} $D_g$, defined by choosing a strictly convex $C^2$ function $g$ and writing
\begin{equation}
D_g(x,x') = g(x) - g(x') - \langle \nabla g(x'), x - x' \rangle.
\end{equation}
To explain how these methods are related to the natural gradient, assume that we are minimizing a general similarity measure $c(x,y)$ with respect to $x$, as in Sec. \ref{sec:formal_natural_gradient}. Recall, that we first defined the natural gradient as a \emph{trust region step}. In order to derive an analytical expression for the iteration, we approximated the objective function with the first order Taylor polynomial and the constraints by the local Hessian and then used Lagrangian duality to yield a \emph{proximal expression}, which yields the formal natural gradient when solved. In Sec. \ref{sec:computations}, we will show how this workflow indeed corresponds to known examples of the natural gradient.

\vspace{10pt}\hspace{-15pt}\textbf{Further connections.} Raskutti and Mukherjee~\cite{raskutti15} showed, that Bregman divergence proximal mirror gradient descent is equivalent to the natural gradient method on the \emph{dual manifold} of the Bregman divergence. Khan et al.~\cite{khan15}, consider a KL divergence proximal algorithm for learning \emph{conditionally conjugate exponential families}, which they show to correspond to a natural gradient step. For exponential families, the KL divergence corresponds to a Bregman divergence, and so the natural gradient step is on the \emph{primal manifold} of the Bregman divergence. Thus the result seems to conflict with the resut in \cite{raskutti15}. However, this can be explained, as the gradient is taken with respect to a different argument of the divergence, i.e., they consider $\nabla_x D_g(x',x)$ and not $\nabla_x D_g(x,x')$. It is intriguing how two different geometries are involved in this choice.

Pascanu and Bengio \cite{pascanu13} remarked on the connections between the natural gradient method and Hessian-free optimization \cite{martens10}, Krylov Subspace Descent \cite{saad81}, and TONGA \cite{roux08}.  The main connection between Hessian-free optimization and Krylov subspace descent is the use of \emph{extended Gauss-Newton approximation of the Hessian} \cite{schraudolph02}, which gives a similar square form involving the Jacobian as the \emph{pullback} Fisher-Rao metric on a statistical manifold. The connection was further studied by Martens \cite{martens14}, where an equivalence criterion between the Fisher-Rao natural gradient and extended Gauss-Newton was given.

\section{Example Computations}
\label{sec:computations}
We now show that in the KL divergence case and the case of a Riemannian metric, the definition of the formal natural gradient matches the classical definition. We contribute local Hessians for formal natural gradients under general $f$-divergences and Finsler metrics, including the $p$-Wasserstein metrics.

\vspace{10pt}\hspace{-15pt}\textbf{Natural gradient of f-divergences.} Let $\rho, \rho' \in\mathrm{AC}(X)$ and $f:\mathbb{R}_{>0}\rightarrow \mathbb{R}_{\geq 0}$ be a convex function satisfying $f(1) = 0$. Then, the $f$-divergence from $\rho'$ to $\rho$ is
\begin{equation}
D_f(\rho ||\rho') = \int_X \rho(x) f\left(\frac{\rho'(x)}{\rho(x)}\right)dx.
\label{def:f_div}
\end{equation}
Now, consider the statistical manifold $(\mathbb{R}^d, \Theta, \rho)$, and compute the local Hessian
\begin{equation}
\left[H^{D_f}_\theta\right]_{ij} = \nabla^2f(1)\int_{X}\frac{\partial \log \rho_\theta(x)}{\partial \theta_i}\frac{\partial \log \rho_\theta(x)}{\partial \theta_j} \rho_\theta(x) dx.
\label{eq:f_div_lhessian}
\end{equation}
Substituting $f = - \log$ in \eqref{def:f_div} results in the KL-divergence, denoted by $D_{\mathrm{KL}}(\rho ||\rho')$. Noticing that $\nabla^2f(1)=1$ with this substitution, we can write \eqref{eq:f_div_lhessian} as $H^{D_f}_\theta = \nabla^2 f(1)H^{D_{\mathrm{KL}}}_\theta$, where the local Hessian $H^{D_{\mathrm{KL}}}_\theta$ is also the Fisher-Rao metric tensor at $\theta$, and thus the natural gradient of Amari \cite{amari98} is retrieved.

\vspace{10pt}\hspace{-15pt}\textbf{Natural gradient of Riemannian distance.} Let $(M,g)$ be a Riemannian manifold with the induced distance function $d_g$ and the metric tensor at $\rho \in M$ denoted by $G^M_\rho$. Finally, denote by $\rho_\theta$ a submanifold of $M$ parametrized by $\theta \in \Theta$. Then, when $c=\frac{1}{2}d^2$, we compute $G_\theta^{\frac{1}{2}d_g}$ as follows
\begin{equation}
\begin{aligned}
\left[H^{\frac{1}{2}d^2}_\theta\right]_{ij}=&\frac{1}{2}\left(\frac{\partial}{\partial \theta_j}\rho_\theta\right)^T \left[\nabla^2_{\rho_\eta \to \rho_\theta}d^2(\rho_\eta,\rho_{\theta})\right]\left(\frac{\partial}{\partial \theta_i}\rho_\theta\right)  \\
&+\frac{1}{2}\left[\frac{\partial^2}{\partial \theta_j \partial \theta_i}\rho_\theta\right]\left[\nabla_{\rho_\eta \to \rho_\theta}d^2(\rho_\eta,\rho_{\theta})\right],
\end{aligned}
\label{eq:distance_hessian}
\end{equation}
as $\theta' \to \theta$, the second term vanishes. Finally, $\nabla^2_{\rho_\eta \to \rho_\theta} d^2(\rho_\eta, \rho_{\theta}) = 2G^M_{\rho_\theta}$, thus
\begin{equation}
H^{\frac{1}{2}d_g}_{\theta} =J_\theta^T G^M_{x_\theta} J_\theta,
\end{equation}
where $J_\theta = \frac{\partial}{\partial \theta}\rho_\theta$ denotes the Jacobian. Therefore, the formal natural gradient corresponds to the traditional coordinate-free definition of a gradient on a Riemannian manifold, when the metric is given by the pullback

\vspace{10pt}\hspace{-15pt}\textbf{Natural gradient of Finsler distance.} Let $(M,F)$ denote a Finsler manifold,  where $F_\rho:T_\rho M \to \mathbb{R}_{\geq 0}$, for any $\rho \in M$, is a \emph{Finsler metric}, satisfying the properties of strong convexity, positive 1-homogeneity and positive definiteness. Then, a distance $d_F$ is induced on $M$ by
\begin{equation}
d_F(\rho,\rho') = \inf\limits_{\gamma} \int_0^1 F_{\gamma(t)}(\dot{\gamma}(t))dt,~\rho,\rho'\in M
\end{equation}
where $\gamma$ is any continuous, unit-parametrized curve with $\gamma(0) = \rho$ and $\gamma(1) = \rho'$. 

The \emph{fundamental tensor} $G^F$ of $F$ at $(\rho,v)$ is defined as $G^F_{\rho}(v) = \frac{1}{2}\nabla^2_{v} F^2_\rho(v)$. Then,  $G^F_\rho$ is $0$-homogeneous as the second differential of a $2$-homogeneous function. Therefore, $G^F_\rho(\lambda v) = G^F_\rho(v)$ for any $\lambda > 0$. Furthermore, $G^F_\rho(v)$ is positive-definite when $v\neq 0$. Now, let $u = -J_\theta \nabla_\theta d^2_F(\rho_\theta, \rho') $, and as we can locally write $d^2_F(\rho, \rho') = F^2_{\rho\theta}(v)$ for a suitable $v$, then
\begin{equation}
H^{\frac{1}{2}d_F^2}_\theta =\frac{1}{2}\nabla^2_{\eta \to \theta} d^2_F(\rho_\eta, \rho_{\theta}) = \frac{1}{2}\lim\limits_{\lambda \to 0} \nabla^2_{v= \lambda u} F^2_{\rho_\theta}(v) = J_\theta^T G^F_{\rho_\theta}(u) J_\theta.
\end{equation}

Coordinate-free gradient descent on Finsler manifolds has been studied by Bercu \cite{bercu00}. The formal natural gradient differs slightly from this, as we choose $v =-J_\theta \nabla_\theta d^2_F(\rho_\theta,\rho')$ in the preconditioning matrix $G^F_{(\rho_\theta, v)}$, where as in \cite{bercu00}, $v$ is chosen to maximize the descent. Thus the natural gradient descent in the Finsler case approximates the geometry in the direction of the gradient quadratically to improve the descent, but fails to take the entire local geometry into account.

\vspace{10pt}\hspace{-15pt}\textbf{$p$-Wasserstein metric.}
Let $X=\mathbb{R}^n$ and $\rho \in \mathcal{P}_p(X)$ if
\begin{equation}
\int_{X}d_2^p(x_0,x)\rho(x)dx,~\mbox{for some }x_0\in X,
\end{equation}
where $d_2$ is the Euclidean distance. Then, the $p$-Wasserstein distance $W_p$ between $\rho, \rho' \in \mathcal{P}_p(X)$ is given by
\begin{equation}
W_p(\rho, \rho') = \left(\inf\limits_{\gamma \in \mathrm{ADM}(\rho, \rho')}\int_{X\times X}d_2^p(x,x')d\gamma(x,x')\right)^\frac{1}{p},
\end{equation}
where $\mathrm{ADM}(\rho, \rho')$ is the set of joint measures with marginal densities $\rho$ and $\rho'$. The $p$-Wasserstein distance is induced by a Finsler metric~\cite{agueh12}, given by
\begin{equation}
F_\rho(v) = \left(\int_X \|\nabla \Phi_v\|_2^p d\rho\right)^\frac{1}{p},
\end{equation}
where $v\in T_\rho \mathcal{P}_p(X)$ and $\Phi_v$ satisfies $v(x) = -\nabla \cdot \left(\rho(x) \nabla_x \Phi_v(x)\right)$ for any $x\in X$, wher $\nabla \cdot$ is the divergence operator. Now, choose $v = -J_\theta\nabla_\theta W_p^2(\rho_\theta, \rho)$. Then, through a cumbersome computation, we compute how the local Hessian acts on two tangent vectors $d\theta_1, d\theta_2\in T_\theta \Theta$
\begin{equation}
\begin{aligned}
&H^{\frac{1}{2}W_p^2}_\theta(d\theta_1, d\theta_2) \\
=&(2-p)F^{2(1-p)}_{\rho_\theta}(v)\left(\int_X \|\nabla \Phi_v\|_2^{p-2}\langle \nabla \Phi_{d\theta_1},\nabla \Phi_v \rangle d\rho_\theta\right)\\
&\times \left(\int_X \|\nabla \Phi_v\|_2^{p-2}\langle \nabla \Phi_{d\theta_2}, \nabla \Phi_v \rangle d\rho_\theta\right)\\
&+ F_{\rho_\theta}^{2-p}(v)\int_X\|\nabla \Phi_v\|_2^{p-2}\langle \nabla \Phi_{d\theta_1},\nabla \Phi_{d\theta_2} \rangle d\rho_\theta\\
&+(p-2)F_{\rho_\theta}^{2-p}(v)\int_X\|\nabla \Phi_v\|_2^{p-4}\langle \nabla \Phi_{d\theta_1},\nabla \Phi_v \rangle\langle \nabla \Phi_{d\theta_2},\nabla \Phi_v \rangle d\rho_\theta,
\end{aligned}
\label{eq:finslerian_lhessian}
\end{equation}
where $J_\theta d\theta_i  = - \nabla \cdot \left(\rho_\theta \nabla \Phi_{d\theta_i}\right)$ for $i=1,2$.
The case $p=2$ is special, as the $2$-Wasserstein metric is induced by a Riemannian metric, whose pullback can be recovered by substituting $p=2$ in \eqref{eq:finslerian_lhessian}, yielding
\begin{equation}
H^{\frac{1}{2}W_2^2}_{\theta}(d\theta_1, d\theta_2) = \int_{X} \langle \nabla\Phi_{d\theta_1}, \nabla \Phi_{d\theta_2}\rangle d\rho_\theta.
\end{equation}
This yields the natural gradient of $W_2^2$ as introduced in \cite{chen18, li18}.

\vspace{10pt}\hspace{-15pt}\textbf{Acknowledgements.} The authors were supported by Centre for Stochastic Geometry and Advanced Bioimaging, and a block stipendium, both funded by a grant from the Villum Foundation.

%\end{comment}

\bibliographystyle{splncs04}
\bibliography{gsi_reference.bib}
\end{document}